\documentclass[11pt,a4paper]{article}
\usepackage{times,latexsym}
\usepackage{url}
\usepackage[T1]{fontenc}

\usepackage[acceptedWithA]{tacl2021v1}
\usepackage{xspace,mfirstuc,tabulary}

\usepackage{tikz}
\usepackage{booktabs} 
\usepackage{amssymb}
\usepackage{latexsym}
\usepackage{amsmath}
\usepackage{amsthm}

\usepackage{pgfplots}
\pgfplotsset{compat=1.18}

    \theoremstyle{definition}
    \newtheorem{dfn}{Definition}[section]
    
    \newtheorem{lem}[dfn]{Lemma}
    \newtheorem{thm}[dfn]{Theorem}

\newif\iftaclinstructions
\taclinstructionsfalse 
\iftaclinstructions
\renewcommand{\confidential}{}
\renewcommand{\anonsubtext}{(No author info supplied here, for consistency with
TACL-submission anonymization requirements)}
\newcommand{\instr}
\fi

\iftaclpubformat 

\else

\fi

\title{From Formal Language Theory to Statistical Learning:\\ Finite Observability of Subregular Languages}

\author{
  Katsuhiko Hayashi
  \\
  Language and Information Sciences \\
  The University of Tokyo 
  \\
  \texttt{{\small katsuhiko-hayashi@g.ecc.u-tokyo.ac.jp}}
  \And
  Hidetaka Kamigaito 
  \\
  Information Science \\
  Nara Institute of Science and Technology
  \\
  \texttt{{\small kamigaito.h@is.naist.jp}}
}

\date{}

\begin{document}
\maketitle
\begin{abstract}
We prove that all standard subregular language classes are linearly separable when represented by their deciding predicates. This establishes finite observability and guarantees learnability with simple linear models. Synthetic experiments confirm perfect separability under noise-free conditions, while real-data experiments on English morphology show that learned features align with well-known linguistic constraints. These results demonstrate that the subregular hierarchy provides a rigorous and interpretable foundation for modeling natural language structure. 
Our code used in real-data experiments is available at \url{https://github.com/{Anonymous}}.

\end{abstract}

\section{Introduction}
\label{sec:intro}

The relationship between formal language theory and natural language has long been a subject of investigation, dating back to the pioneering work of \citet{chomsky1963}, which connected context-free languages to algebraic structures.
This tradition has demonstrated that the formal properties of natural language are not arbitrary but are tightly constrained, often residing within computationally restricted subclasses of the regular languages.

In recent decades, research has increasingly focused on the \emph{subregular hierarchy}, a set of language classes strictly contained within the regular languages but capable of capturing a wide range of phonological, morphological, and even syntactic phenomena \citep{heinz2010,heinz2011tsl,graf2022subregular,graf2022diving,torres2023modeling,hanson2024strict}.
This line of work has revealed that linguistic constraints often fall within these low-complexity classes, explaining their learnability and robustness in acquisition \citep{heinz2011idsardi,saffran1996}.
For example, phonotactic restrictions such as vowel harmony, consonant dissimilation, and tonal patterns have been successfully modeled with subregular languages \citep{heinz2010,jardine2016}.
Morphological and syntactic dependencies have likewise been shown to align with subregular complexity \citep{chandlee2014}.

While the algebraic and logical characterizations of subregular classes are well established \citep{straubing2012,thomas1997,pin1996logic}, there has been limited work providing a \emph{geometric} perspective on why these classes matter for learnability.
At the same time, statistical learning theory has established powerful results connecting \emph{linear separability} to generalization guarantees via the VC dimension, perceptron mistake bounds, and margin-based analyses \citep{vapnik1999overview,blumer1989,novikoff1962,cortes1995,shalev2014}.
These results suggest that if subregular classes can be shown to be linearly separable in appropriate finite feature spaces, then we can directly apply statistical learning theory to natural language phenomena.

In this paper, we introduce the notion of \emph{finite observability} and prove that all standard subregular classes are finitely observable, which entails that they are linearly separable under suitable embeddings.
This provides a novel \emph{geometric characterization} of subregular languages, complementing existing algebraic and automata-theoretic accounts \citep{mcnaughton1971,schutzenberger1965,eilenberg1974automata}.
By embedding strings into finite-dimensional Boolean spaces defined by primitive predicates, we unify disparate subregular classes within a single framework and open the door to direct application of statistical learning guarantees.

The contributions of this paper are threefold:
\begin{enumerate}
    \item We define \textbf{finite observability}, a property ensuring that membership in a language depends only on finitely many primitive observations (predicates).
    \item We prove that all standard subregular classes ($\mathsf{SL}$, $\mathsf{SP}$, $\mathsf{LT}$, $\mathsf{PT}$, $\mathsf{LTT}$, $\mathsf{TSL}$) are finitely observable (see Appendix~\ref{app:finite-observability-subregular}).
    \item We show that finite observability entails \textbf{linear separability} in Boolean feature spaces, thereby bridging formal language theory, statistical learning theory, and linguistic applications.
\end{enumerate}

This perspective enriches our theoretical understanding of why natural language patterns are learnable, while also suggesting practical methods for interpretable, efficient models of linguistic constraints.

\section{Related Work}
\label{sec:related}

\paragraph{Subregular hierarchy: historical and theoretical foundations.}
The subregular hierarchy has been established as a family of language classes that are both \emph{computationally weaker} than regular languages and yet \emph{linguistically expressive} enough to capture many phonological and morphological constraints.
The algebraic characterization of \emph{star-free languages} and their correspondence with first-order logic FO\mbox{[}$<$\mbox{]} laid much of the theoretical foundation \citep{mcnaughton1971,straubing2012,schutzenberger1965,eilenberg1974automata,thomas1997}.
In particular, Simon's characterization of \emph{Piecewise Testable $\mathsf{(PT)}$} languages \citep{simon1975} and the algebraic and logical perspectives on \emph{Locally Testable $\mathsf{(LT)}$} languages \citep{mcnaughton1971,straubing2012,thomas1997} illustrate how language membership can be determined by \emph{finite local observations}, thereby connecting naturally to $\mathsf{SL}_k$ and $\mathsf{SP}_k$.
Tier-based Strictly Local ($\mathsf{TSL}$) languages, which localize nonlocal dependencies via tier projection, have proven to be a powerful extension \citep{heinz2011tsl}.
Our contribution complements this tradition by providing a \emph{geometric} characterization of subregular classes via linear separability, offering an alternative perspective to classical algebraic and automata-theoretic accounts.

\paragraph{Learnability, grammatical inference, and statistical learning theory.}
The grammatical inference literature has long studied the identifiability of language classes, beginning with Gold's identification-in-the-limit framework \citep{gold1967}, Angluin's query learning paradigm \citep{angluin1987}, and consolidated in de~la~Higuera's monograph \citep{delahiguera2010}.
Early work showed that certain subclasses of regular languages (e.g., $k$-testable in the strict sense) admit efficient inference algorithms \citep{garcia1990,oncina1992}.
For subregular languages, numerous learnability results have been established for $\mathsf{SL}$, $\mathsf{SP}$, and $\mathsf{TSL}$ classes \citep{heinz2010,heinz2011tsl,jardine2016,chandlee2014}.
Meanwhile, statistical learning theory provides a different angle: linear separability yields strong guarantees via the VC dimension \citep{vapnik1999overview,blumer1989,shalev2014}, the perceptron mistake bound \citep{novikoff1962}, and SVM margin analyses \citep{cortes1995}.
Our contribution demonstrates that subregular classes are linearly separable in appropriate finite feature spaces, thereby allowing these guarantees to apply directly to natural language classes themselves. This unifies classical grammatical inference with modern statistical learning theory.

\section{Method}
\label{sec:method}

This section introduces the two central notions used throughout the paper: \emph{predicates} and \emph{finite observability}, and develops our main theorem that finite observability entails linear separability under an explicit embedding.
We also explain the learning-theoretic consequences and provide practical guidance for using low-dimensional predicate features in applications.
All formal proofs are deferred to \S\ref{sec:appendix-proofs}.

\subsection{Predicates and Finite Observability}
\label{sec:predicates}

Let $\Sigma$ be a finite alphabet and let $\#$ be a boundary symbol not in $\Sigma$.
For $K\in\mathbb{N}$ and $x\in\Sigma^*$, we write $\tilde{x}=\#^{K}\,x\,\#^{K}$ for the boundary-extended string.
A \emph{predicate} is a boolean-valued function
\begin{equation}
  p:\Sigma^* \to \{0,1\},
\end{equation}
which detects a finite, well-defined property of $x$ (e.g., presence of a particular substring, subsequence, or a thresholded count).
We emphasize that predicates are the \emph{atoms of observation}: once we know which predicates hold of $x$, no additional information about $x$ is required for membership decisions in finitely observable classes.

\paragraph{Finite Observability.}
A language class $\mathcal{C}\subseteq\mathcal{P}(\Sigma^*)$ is \emph{finitely observable} if there exists a \emph{finite} set of predicates $P=\{p_1,\dots,p_n\}$ such that, for every $L\in\mathcal{C}$, there is a set $S_L\subseteq\{0,1\}^n$ with
\begin{equation}
  \label{eq:finite-observability}
  x\in L \;\Longleftrightarrow\; (p_1(x),\dots,p_n(x))\in S_L \qquad \forall x\in\Sigma^*.
\end{equation}
Thus membership in any language $L$ depends only on the \emph{truth vector} of finitely many primitive predicates.

\paragraph{Illustrative predicates (examples).}
Below we list the predicate types that will serve as building blocks for standard subregular families: Appendix~\ref{app:finite-observability-subregular} gives formal constructions and proofs.
\begin{itemize}
  \item \textbf{Substring predicates ($\mathsf{SL}_k$)}:
    for $g\in(\Sigma\cup\{\#\})^k$,
    $p_g(x)=\mathbf{1}[\,g\text{ occurs contiguously in }\tilde{x}\,]$.
    \emph{Example:} in an SL$_3$ language forbidding 3-gram \texttt{ngt}, the constraint is captured by requiring $p_{\texttt{ngt}}(x)=0$.
  \item \textbf{Subsequence predicates ($\mathsf{SP}_k$)}:
    for $h\in(\Sigma\cup\{\#\})^k$,
    $q_h(x)=\mathbf{1}[\,h\text{ occurs as a subsequence of }\tilde{x}\,]$.
    \emph{Example:} in SP$_2$, $q_{\texttt{a\_i}}(x)$ indicates whether \texttt{a} precedes \texttt{i}; forbidding such pairs encodes vowel-harmony style constraints.
  \item \textbf{Boundary predicates ($\mathsf{LT}_k$)}:
    for $u,v\in(\Sigma\cup\{\#\})^{k-1}$,
    $\pi_u(x)=\mathbf{1}[\,\mathrm{prefix}_{k-1}(\tilde{x})=u\,]$,
    $\sigma_v(x)=\mathbf{1}[\,\mathrm{suffix}_{k-1}(\tilde{x})=v\,]$.
  \item \textbf{Thresholded count predicates ($\mathsf{LTT}_{k,\boldsymbol\tau}$)}:
    for $g$ with $|g|\le k$ and $t\in\{1,\dots,\tau_g\}$,
    $c_{g,t}(x)=\mathbf{1}[\,\#\{g\text{ in }\tilde{x}\}\ge t\,]$.
  \item \textbf{Tiered predicates ($\mathsf{TSL}_k$)}:
    for a tier $T\subseteq\Sigma$ and $g\in (T\cup\{\#\})^k$,
    $s_{T,g}(x)=\mathbf{1}[\,g\text{ occurs in the projection }\pi_T(\tilde{x})\,]$,
    where $\pi_T$ erases symbols outside $T$.
\end{itemize}

\noindent
Despite their different linguistic motivations, these predicates are all \emph{finite} in number for fixed class parameters ($k$, thresholds, and the finite alphabet), and they suffice to decide membership across the corresponding subregular families.

\subsection{Minterm Linearization}
\label{sec:minterm}

Given predicates $P=\{p_1,\dots,p_n\}$, define the \emph{truth vector}
$r(x)=(p_1(x),\dots,p_n(x))\in\{0,1\}^n$.
For each minterm $a\in\{0,1\}^n$, define the \emph{minterm indicator}
\begin{equation}
  \label{eq:minterm}
  \mu_a(x) \;=\; \prod_{j=1}^n \big( p_j(x)^{a_j}\,(1-p_j(x))^{1-a_j} \big) \in \{0,1\}.
\end{equation}
By construction, $\mu_a(x)=1$ iff $r(x)=a$; otherwise $\mu_a(x)=0$.
Let the \emph{minterm feature map} be
\begin{equation}
  \Phi:\Sigma^*\to\{0,1\}^{2^n},\qquad \Phi(x) = \big(\mu_a(x)\big)_{a\in\{0,1\}^n}.
\end{equation}

\begin{lem}[Minterm Linearization]\label{lem:minterm}
For any $S\subseteq\{0,1\}^n$, there exist weights $w\in\mathbb{R}^{2^n}$ and bias $b\in\mathbb{R}$ such that
\begin{equation}
  \langle w,\Phi(x)\rangle + b > 0 \iff r(x)\in S
  \qquad\text{for all }x\in\Sigma^*.
\end{equation}
Moreover, the separating hyperplane attains functional margin at least $1/2$;
the geometric margin equals $1/(2\|w\|_2)$.
\end{lem}

\noindent
\emph{Intuition.} Each minterm indicator is a one-hot basis vector that uniquely identifies the truth assignment $r(x)$. Any boolean decision rule over $r(x)$, including non-monotone combinations such as XOR, is representable by selecting the corresponding minterm coordinates and thresholding. A short proof appears in Appendix~\ref{app:proof-minterm}.

\subsection{Main Theorem: Finite Observability $\Rightarrow$ Linear Separability}
\label{sec:main-theorem}

\begin{thm}\label{thm:fo-to-linear}
Let $\mathcal{C}$ be a finitely observable class with predicates $P=\{p_1,\dots,p_n\}$.
Then for every $L\in\mathcal{C}$ there exist $(w,b)$ such that
\begin{equation}
  x\in L \iff \langle w,\Phi(x)\rangle + b > 0 \qquad\text{for all }x\in\Sigma^*,
\end{equation}
where $\Phi$ is the minterm feature map in \eqref{eq:minterm}. The margin is at least $1/2$.
\end{thm}

\noindent
\emph{Explanation.} By finite observability, there exists $S_L\subseteq\{0,1\}^n$ such that \eqref{eq:finite-observability} holds.
Lemma~\ref{lem:minterm} constructs a single hyperplane that separates minterms in $S_L$ from those outside $S_L$. Composing with $\Phi$ yields the desired linear separator for $L$.
See Appendix~\ref{app:proof-main} for a formal proof.

\paragraph{Learning-theoretic consequences.}
Each feature vector $\Phi(x)$ is one-hot with $\|\Phi(x)\|_2=1$, so the effective radius is $R=1$ and the geometric margin is $\gamma\ge 1/(2\|w\|_2)$.
Hence a perceptron run on \emph{separable} data suffers at most $M\le (1/\gamma)^2\le 4\|w\|_2^2=4|S|$ mistakes \citep{novikoff1962}.
Moreover, the class of halfspaces in $\mathbb{R}^{2^n}$ has VC dimension at most $2^n$ \citep{blumer1989}, and margin-based generalization bounds apply directly \citep{cortes1995,shalev2014}.
In short, finitely observable classes are \emph{easy to learn} under standard linear methods.

\subsection{From Theory to Practice: Low-dimensional Predicate Features}
\label{sec:practice}

Although the minterm embedding is conceptually clean, it is exponential in $n$.
In practice, the subregular families admit \emph{direct} low-dimensional predicate features that already separate the target languages:
\begin{itemize}
  \item \textbf{$\mathsf{SL}_k$}: presence/absence of $k$-grams (with boundary padding) suffices to implement forbidden-substring constraints.
  \item \textbf{$\mathsf{SP}_k$}: presence/absence of $k$-length subsequences suffices to implement forbidden nonadjacent pairs/tuples.
  \item \textbf{$\mathsf{LT}_k$}: $k$-gram features augmented with $(k{-}1)$-length prefix/suffix indicators determine membership.
  \item \textbf{$\mathsf{PT}_m$}: presence/absence of subsequences of length $\le m$~\citep{simon1975}.
  \item \textbf{$\mathsf{LTT}_{k,\boldsymbol\tau}$}: bounded-threshold counts for substrings of length $\le k$.
  \item \textbf{$\mathsf{TSL}_k$}: $k$-gram features on tier projections $\pi_T$ (for $T\subseteq\Sigma$) capture nonlocal constraints through local ones on the tier \citep{heinz2011tsl}.
\end{itemize}
These directly usable features retain the interpretability of symbolic constraints and allow linear classifiers (logistic regression, linear SVMs, perceptrons) to achieve exact separation on the intended targets while remaining computationally lightweight.

\paragraph{Scope and limitations.}
Our results require fixing class parameters (e.g., $k$, $m$, $\boldsymbol\tau$) and the finite alphabet, which is standard when defining target language families.
Unions over unbounded parameters (e.g., $\mathsf{PT}$ $=\bigcup_m \mathsf{PT}_m$) should be viewed as directed unions of finitely observable subclasses (Appendix~\ref{app:finite-observability-subregular}).
While the minterm embedding itself is exponential, practical predicate sets for the standard subregular families are polynomial in $|\Sigma|$ and the fixed parameters, and can be trained with sparse linear methods.

\section{A Counterexample Showing the Limits of Fixed Finite Predicate Sets}
\label{sec:analysis-regular-counterexample}

Our main results hinge on \emph{finite observability}: once a finite set of predicates $P=\{p_1,\dots,p_n\}$ is fixed, any language definable in this framework is a union of the finitely many \emph{predicate cells} (truth assignments) in $\{0,1\}^n$.
This section shows that, for any \emph{fixed} finite predicate set $P$, there exist regular languages that cannot be detected (i.e., represented as unions of predicate cells), hence cannot be linearly separated in the corresponding feature space.
The proof relies on the Myhill--Nerode theorem \citep{sipser1996introduction}.

\paragraph{Setup.}
Fix any finite alphabet $\Sigma$ and any finite predicate set $P=\{p_1,\dots,p_n\}$ with $n\ge 1$.
Let $r(x)=(p_1(x),\dots,p_n(x))\in\{0,1\}^n$.
The induced equivalence relation $\equiv_P$ on $\Sigma^*$ groups two strings iff they share the same truth vector: $x\equiv_P y \iff r(x)=r(y)$.
There are at most $2^n$ such equivalence classes (\emph{cells}).
Any language expressible via $P$ is necessarily a union of $\equiv_P$-classes.

\begin{thm}[Regular language not representable by a fixed $P$]
\label{thm:regular-not-captured}
For any fixed finite predicate set $P$ over $\Sigma$, there exists a regular language $L$ such that $L$ is not a union of $\equiv_P$-classes.
Equivalently, no linear classifier over the fixed predicate feature space can separate $L$ from its complement.
\end{thm}

\begin{proof}
Let $m>2^n$ be any integer and consider the one-letter alphabet $\Sigma=\{a\}$.
Define the regular language
\[
  L_m \;=\; \{\, a^{k} \mid k \equiv 0 \pmod m \,\}.
\]
It is well known (and follows from Myhill--Nerode) that $L_m$ has exactly $m$ distinct right-congruence classes, corresponding to length modulo $m$ \citep{sipser1996introduction}.
For any equivalence relation $E$ whose index (number of classes) is strictly smaller than $m$, $E$ cannot refine the Myhill--Nerode congruence of $L_m$; hence $L_m$ cannot be a union of $E$-classes.
Since $\equiv_P$ has at most $2^n$ classes and $m>2^n$, $\equiv_P$ does not refine the Myhill--Nerode congruence of $L_m$.
Therefore $L_m$ is not a union of $\equiv_P$-classes and cannot be represented in the predicate framework determined by $P$.
Because linear classification over $P$ can only realize unions of $\equiv_P$-cells, no linear separator exists for $L_m$ in that fixed feature space.
\end{proof}

\paragraph{Implications.}
Theorem~\ref{thm:regular-not-captured} formalizes an intuitive limitation: fixing a finite observation vocabulary $P$ yields at most $2^n$ observable profiles, while regular languages can require arbitrarily many distinct right-congruence classes (e.g., $m$ for $L_m$), even over a unary alphabet.
Thus, unlike the subregular families discussed in this paper, the class of all regular languages cannot be captured~(nor linearly separated) within any single finite predicate space.
\section{Synthetic Experiments}
\begin{figure*}[t]
\centering
\begin{tabular}{lll}
A)~$\mathsf{SL}_3$~(noise) & B)~$\mathsf{SP}_2$~(noise) & C)~$\mathsf{LTT}_2$~(noise) \\
\includegraphics[scale=0.33]{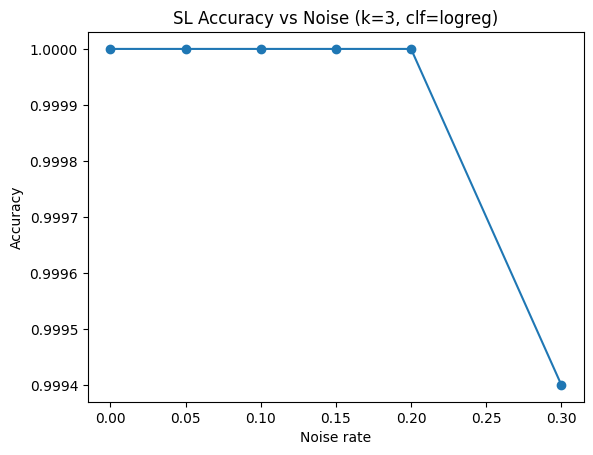} &
\includegraphics[scale=0.33]{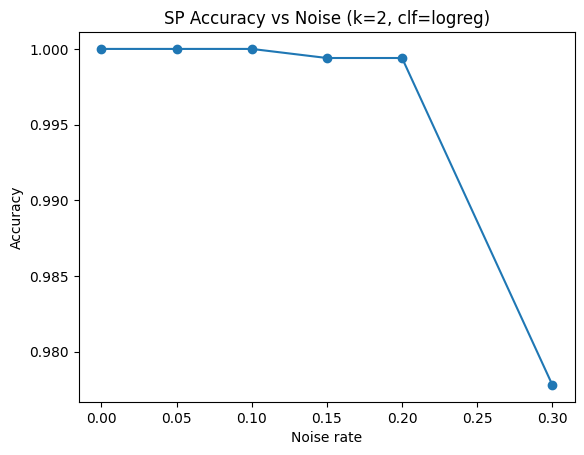} &
\includegraphics[scale=0.33]{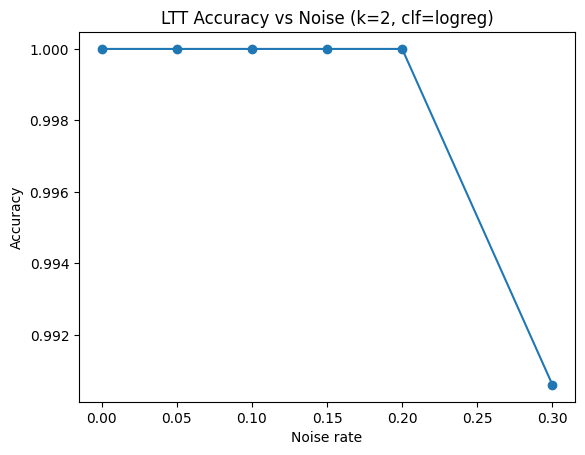} \\
D)~$\mathsf{SL}_3$~(train size) & E)~$\mathsf{SP}_2$~(train size) & F)~$\mathsf{LTT}_2$~(train size) \\
\includegraphics[scale=0.33]{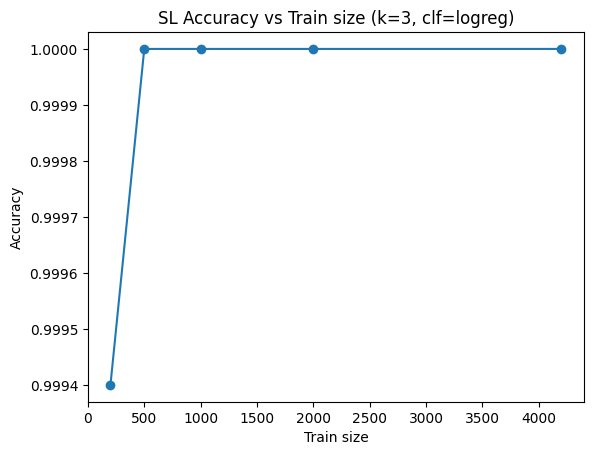} &
\includegraphics[scale=0.33]{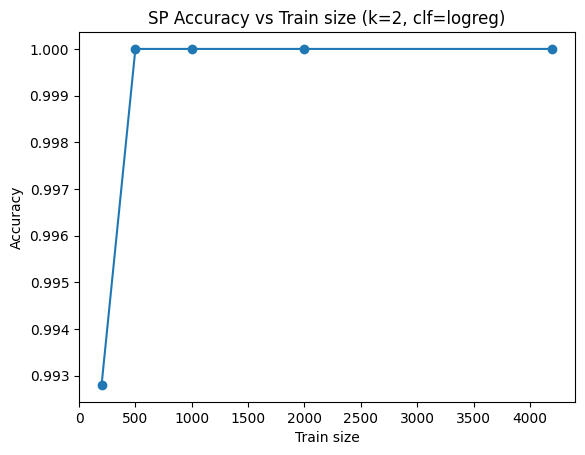} &
\includegraphics[scale=0.33]{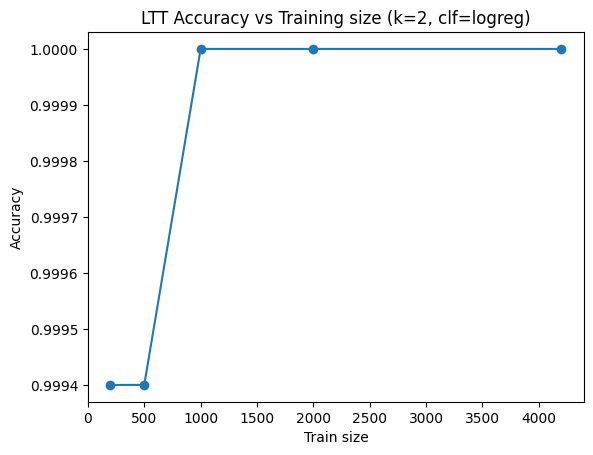}
\end{tabular}
\caption{Accuracy of synthetic classification: (A) $\mathsf{SL}_3$, (B) $\mathsf{SP}_2$ and (C) $\mathsf{LTT}_2$ under label noise, (D) $\mathsf{SL}_3$, (E) $\mathsf{SP}_2$ and (F) $\mathsf{LTT}_2$ with varying training sizes.
 }
\label{fig:SLSPLTT}
\end{figure*}
\subsection{Linear Separability}
To empirically validate our theoretical results, we conducted controlled experiments 
on artificially generated languages. The goal of these experiments is not to maximize 
accuracy itself, but rather to visualize how linear separability manifests across 
different subregular classes when noise or sample size is varied. 

To keep the presentation concise and focused, we concentrate on three representative 
classes from the subregular hierarchy: 
\emph{Strictly Local ($\mathsf{SL}$)}, \emph{Strictly Piecewise ($\mathsf{SP}$)}, and 
\emph{Locally Threshold Testable ($\mathsf{LTT}$)}. 
$\mathsf{SL}$ captures local adjacency constraints and provides the most transparent connection 
to phonotactic well-formedness. $\mathsf{SP}$ models long-distance restrictions, such as 
co-occurrence bans, which are common in phonology and morphology. Finally, $\mathsf{LTT}$ extends 
local tests with counting thresholds, allowing us to probe more complex phenomena and 
to examine how separability behaves as the constraints become richer. 

For each class, we generate positive and negative examples according to the defining 
predicates, then train logistic regression models using the corresponding predicate 
features. By sweeping noise levels and training set sizes, we observe how empirical 
performance converges toward the theoretical guarantee of linear separability.

\subsubsection{Strictly Local Languages~($\mathsf{SL}_k$)}
\paragraph{Setup.}
We generated synthetic datasets for $\mathsf{SL}_3$ languages.
A set of forbidden 3-grams was defined over a finite alphabet with explicit inclusion of boundary symbols <$\#$>, <$/\#$>.
Positive examples were sampled uniformly from sequences that avoided all forbidden 3-grams, while negative examples were constructed to guarantee the presence of at least one forbidden 3-gram (including those at boundaries). Each string length was sampled randomly within a fixed range (5–15 characters).
The feature representation was the full set of boundary-aware 3-grams over the extended alphabet $\Sigma\cup\{\text{<}\#\text{>},\text{<}/\#\text{>}\}$.
This guarantees that the deciding predicates are included.

\paragraph{Results.}
Figure~\ref{fig:SLSPLTT} (A) shows the accuracy under a noise sweep, where increasing proportions of training labels were randomly flipped. As predicted, performance is exactly 100\% in the noise-free setting, and decreases as noise increases. Figure~\ref{fig:SLSPLTT} (D) shows the size sweep, where training set size was varied. Performance converges rapidly toward 100\% as more data are available, reflecting the finite observability of deciding predicates.

\subsubsection{Strictly Piecewise Languages~($\mathsf{SP}_k$)}
\paragraph{Setup.}
For $\mathsf{SP}_2$ languages, we defined a forbidden set of subsequences of length two (e.g., $(a,c)$, $(b,d)$) over an alphabet of four symbols. Positive examples were strings avoiding these subsequences, while negative examples were constructed to ensure that at least one forbidden subsequence occurred. String lengths were randomly sampled in the range 6–18 characters.
The feature representation was the set of all possible symbol pairs (ordered subsequences) over the alphabet. This again ensures that all deciding predicates are explicitly represented in the feature space.

\paragraph{Results.}
As with $\mathsf{SL}_3$, Figure~\ref{fig:SLSPLTT} (B) shows accuracy under a noise sweep, and Figure~\ref{fig:SLSPLTT} (E) shows results for a size sweep. In the absence of label noise, classification achieves 100\% accuracy, exactly matching theoretical expectations. 
When noise is introduced, performance degrades smoothly, illustrating the robustness of the separating representation.

\subsubsection{Locally 
Threshold Testable Languages~($\mathsf{LTT}_k$)}
\paragraph{Setup.}
For the $\mathsf{LTT}_2$ condition, we defined membership in terms of thresholded counts of substrings with boundaries. Specifically, positive examples were required to contain at least two instances of the symbol $a$, begin with $a$ at the left boundary, contain at most one occurrence of the 2-gram $bb$, and contain at most one occurrence of $c$ at the right boundary. Positive examples were generated by rejection sampling with targeted repairs to ensure that all threshold constraints were satisfied. Negative examples were produced by minimally breaking one constraint, for example by reducing the number of $a$’s, inserting an additional $bb$, altering the prefix so it no longer begins with $a$, or forcing a terminal $c$. This guaranteed that every negative sequence violated at least one deciding predicate.
Feature vectors were constructed from all substrings of length at most two, extended with boundary markers, with each feature indicating whether the count of that substring exceeded a given threshold. These features explicitly encode the deciding predicates of $\mathsf{LTT}_2$.

\paragraph{Results.}
Figure~\ref{fig:SLSPLTT} (C) shows the effect of label noise on classification accuracy, and Figure~\ref{fig:SLSPLTT} (F) shows the effect of training size. In the noise-free setting, accuracy reached 100\%, confirming the theoretical prediction of perfect separability. As noise increased, accuracy decreased, while larger training sizes quickly restored performance toward the ideal. 
These results provide empirical evidence that threshold-based constraints in $\mathsf{LTT}$ languages are finitely observable and linearly separable when represented with appropriate features.

\section{Experiments on Morphological Data}
\paragraph{Setup.}
To test the applicability of our theoretical findings to natural language data, we conducted experiments on English derivational morphology using the MorphoLex-en database~\citep{sanchez2018morpholex}. 
MorphoLex provides a large-scale lexicon of English words annotated with morphological features, including derivational affixes. We extracted affix sequences for each word, consisting of ordered prefixes followed by suffixes. Since MorphoLex does not always provide explicit segmentation, we complemented the data with a heuristic segmenter based on a comprehensive inventory of common English affixes. This yielded a dataset in which each item was represented by a sequence of affixes.

To create a classification task aligned with the subregular framework, we constructed positive examples from well-formed affix sequences in MorphoLex, and generated negative examples by minimally perturbing them. Specifically, negatives were created by either permuting the order of affixes or substituting one affix with another from the same vocabulary. This ensured that negative sequences closely resembled real ones in length and vocabulary but violated natural morphological constraints. The resulting dataset was divided into training, development, and test sets, with splits stratified by word to avoid leakage.

We represented each affix sequence using two subregular feature classes: Piecewise Testable ($\mathsf{PT}$) predicates, capturing all subsequences up to length 
$m=2$, and Locally Threshold Testable ($\mathsf{LTT}$) predicates, encoding boundary-aware 
$k=2$. As in the synthetic experiments, we trained logistic regression classifiers on these feature vectors. Importantly, the classifiers were not constructed directly from the theoretical separating hyperplanes, but were trained from data, demonstrating that linear separability is empirically recoverable by standard learning algorithms.

\paragraph{Results.}
\begin{figure}[t]
\centering
\includegraphics[scale=0.475]{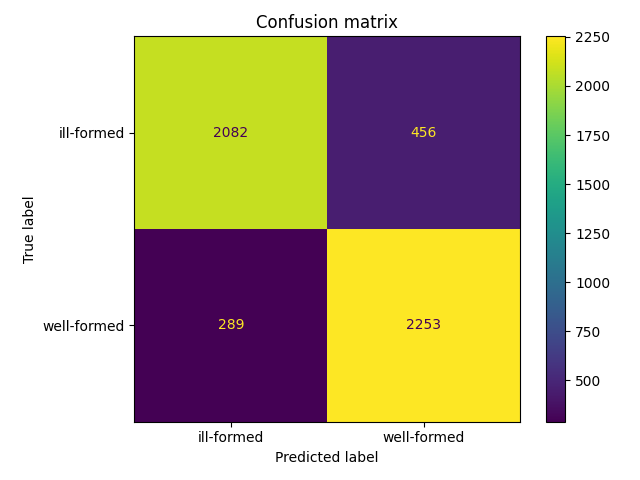}
\caption{Confusion matrix for the morphological well-formedness classification task, showing accurate separation of well-formed (positive) and ill-formed (negative) affix sequences.}
\label{fig:confusion}
\end{figure}
\begin{figure*}[t]
\centering
\includegraphics[scale=0.6]{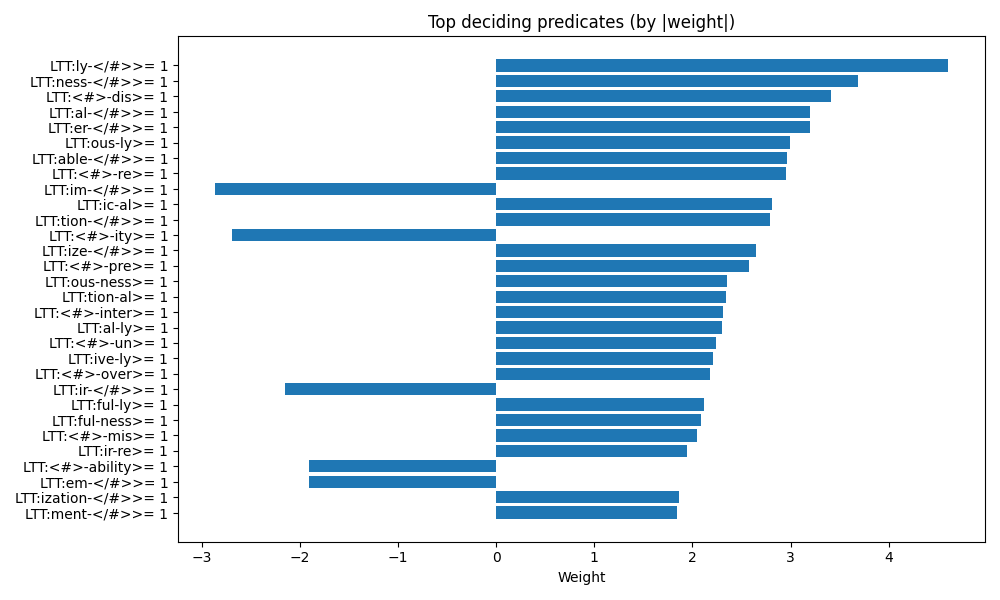}
\caption{Top-weighted predicates learned by the classifier on the morphological dataset, illustrating linguistically meaningful affix constraints such as boundary-sensitive suffixes and prefix–suffix co-occurrence restrictions.}
\label{fig:features}
\end{figure*}
\begin{figure}[t]
\centering
\includegraphics[scale=0.475]{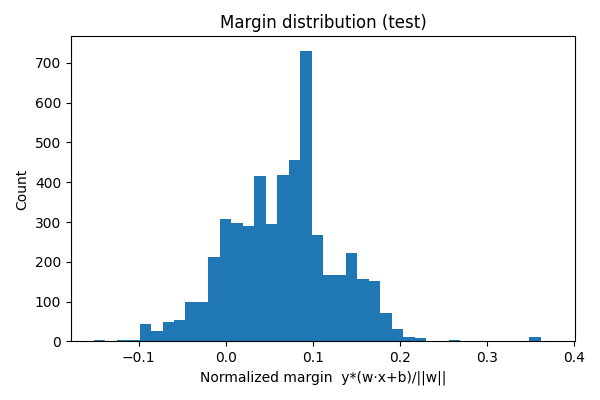}
\caption{Histogram of normalized margins on the test set. Most examples lie at positive margins, confirming effective linear separation.}
\label{fig:margin_hist}
\end{figure}
On the held-out test set, the classifiers achieved strong performance. Logistic regression with $\mathsf{PT+LTT}$ features typically reached accuracy above 85\%, with F1 scores 86\%.
Confusion matrices in Figure~\ref{fig:confusion} showed that more errors occurred in negative items that were close to positive forms in surface structure.
To better understand what the classifier has learned from the real morphological data, we inspected the top-weighted predicates in the logistic regression model (Figure~\ref{fig:features}). Strikingly, the most influential features correspond to linguistically interpretable affix constraints. For instance, the model strongly weights boundary-sensitive predicates such as the presence of {\it -ly} or {\it -ness} at the right edge, reflecting the well-attested distributional restrictions of these suffixes. Similarly, certain prefix–suffix interactions emerge as highly discriminative, such as restrictions involving {\it dis-} or {\it re-} in combination with other affixes.
These results demonstrate that the classifier is not only achieving high predictive accuracy, but is also recovering linguistically meaningful generalizations that align with theoretical expectations in morphology.

\paragraph{Analysis.}
To assess the practical linear separability of morphological patterns, we analyzed the distribution of normalized margins on the held-out test set. For each word, the margin was defined as
$m_i \;=\; y_i \cdot \frac{w^\top x_i + b}{\lVert w \rVert_2}$, where $y_i \in \{+1,-1\}$ is the gold label, $x_i$ the feature vector, $w$ the learned weight vector, and $b$ the intercept. 
A positive margin indicates that the example is correctly classified and lies on the correct side of the separating hyperplane, while the magnitude of the margin reflects the confidence of the classification.

Figure~\ref{fig:margin_hist} shows the histogram of margins for all test examples. The distribution is clearly skewed toward positive values, with the majority of examples clustered around small but positive margins. This indicates that the classifier has successfully learned a separating hyperplane consistent with the morphological constraints. At the same time, the long tail of low and slightly negative margins reveals the presence of ambiguous or difficult cases, such as words with rare affixes or irregular formations. 

The overall pattern supports our theoretical claim: although real data inevitably contains noise and irregularities, the linear predicate-based representation yields a decision boundary that cleanly separates most examples. The margin distribution therefore provides empirical evidence that subregular constraints, when instantiated on lexical data, remain both observable and learnable within a linear framework.

\section{Conclusion}
\begin{table*}[t]
\centering
\begin{tabular}{l l l}
\toprule
Class & Predicate set $P$ (finite) & Rationale \\
\midrule
$\mathsf{SL}_k$ & $k$-gram presence in $\tilde{x}$ & forbidden substrings \\
$\mathsf{SP}_k$ & $k$-subsequence presence in $\tilde{x}$ & forbidden subsequences \\
$\mathsf{LT}_k$ & $k$-grams + $(k{-}1)$ prefix/suffix & local testability \\
$\mathsf{PT}_m$ & subsequences of length $\le m$ & Simon's theorem \\
$\mathsf{LTT}_{k,\boldsymbol\tau}$ & thresholded counts for $|g|\le k$ & bounded counting \\
$\mathsf{TSL}_k$ & $k$-grams on all tiers $\pi_T(\tilde{x})$ & tier projection + $\mathsf{SL}$ \\
\bottomrule
\end{tabular}
\caption{Summary: Finite observability of standard
subregular classes.}
\label{tab:proof}
\end{table*}
This work has shown that all major subregular language classes are linearly separable when expressed through their deciding predicates. This establishes finite observability across the subregular hierarchy and guarantees learnability with simple linear models. Synthetic and morphological experiments confirm these guarantees while also recovering linguistically interpretable constraints, thereby providing a geometric foundation for lightweight, interpretable models that connects naturally to earlier probabilistic approaches such as \citet{hayes2008maximum}'s Maximum Entropy model of phonotactics.

Our perspective contrasts with research on the expressive power of recurrent neural networks (RNNs), which emphasizes their maximal computational capacity, from Turing completeness under infinite precision~\citep{siegelmann1992computational} to pushdown-automaton simulations with LSTMs~\citep{weiss2018practical}. Rather than focusing on what is computationally possible in principle, we identify a restricted region of the hierarchy where natural language patterns actually reside and show that these classes are inherently simple to learn. Future work will examine whether larger families admit analogous notions of observability, and whether the effective capacity of neural models in practice aligns more closely with subregular structure than with their theoretical upper bounds.

\appendix
\section{Appendix: Proofs and Constructions}
\label{sec:appendix-proofs}
\subsection{Notation and Preliminaries}
We write $\mathcal{P}(\Sigma^*)$ for the power set of $\Sigma^*$, i.e., the set of all languages over $\Sigma$.
For a boolean condition $\varphi$, we use $\mathbf{1}[\varphi]$ to denote its indicator in $\{0,1\}$.
The boundary-extended string is $\tilde{x}=\#^{K}x\#^{K}$ with $K$ chosen per construction so that all local observations are well-defined.

\subsection{Proof of Lemma~\ref{lem:minterm} (Minterm Linearization)}
\label{app:proof-minterm}
Let $S\subseteq\{0,1\}^n$, and define weights $w\in\mathbb{R}^{2^n}$ and bias $b\in\mathbb{R}$ by
\begin{equation*}
  w_a=\begin{cases}1 & \text{if } a\in S,\\[2pt] 0 & \text{otherwise},\end{cases}
  \qquad b=-\tfrac12.
\end{equation*}
For any $x$, exactly one minterm indicator equals $1$, namely $\mu_{r(x)}(x)$, and all others are $0$.
Therefore $\sum_{a}w_a\mu_a(x)=\mathbf{1}[\,r(x)\in S\,]$.
Adding $b=-\tfrac12$ yields $+\tfrac12$ on positives and $-\tfrac12$ on negatives, establishing linear separability with margin at least $1/2$. \qed

\subsection{Proof of Theorem~\ref{thm:fo-to-linear}}
\label{app:proof-main}
By finite observability, there exists $S_L\subseteq\{0,1\}^n$ such that $x\in L$ iff $r(x)\in S_L$.
Apply Lemma~\ref{lem:minterm} to $S=S_L$ to obtain $(w,b)$ with
$\langle w,\Phi(x)\rangle+b>0$ iff $r(x)\in S_L$.
This proves the theorem, with margin at least $1/2$ inherited from Lemma~\ref{lem:minterm}. \qed

\subsection{Finite Observability of Standard Subregular Classes}
\label{app:finite-observability-subregular}

We sketch predicate sets witnessing finite observability for each standard family; detailed definitions match those in the main text.
We summarize these results in Table~\ref{tab:proof}.

\paragraph{$\mathsf{SL}_k$.}
Let $P=\{p_g: g\in(\Sigma\cup\{\#\})^k\}$ where $p_g(x)=\mathbf{1}[\,g\prec\tilde{x}\,]$.
A language $L$ is specified by a finite forbidden set $F$ of $k$-grams; membership is equivalent to $p_g(x)=0$ for all $g\in F$.
Thus $S_L=\{v\in\{0,1\}^{|P|}: v_g=0 \ \forall g\in F\}$.

\paragraph{$\mathsf{SP}_k$.}
Let $P=\{q_h: h\in(\Sigma\cup\{\#\})^k\}$ where $q_h(x)=\mathbf{1}[\,h\sqsubseteq \tilde{x}\,]$ (subsequence).
With forbidden subsequences $F$, set $S_L=\{v: v_h=0 \ \forall h\in F\}$.

\paragraph{$\mathsf{LT}_k$.}
Let $P$ contain all $k$-gram indicators together with prefix/suffix indicators of length $k{-}1$:
$p_g(x)=\mathbf{1}[\,g\prec\tilde{x}\,]$, $\pi_u(x)=\mathbf{1}[\,\mathrm{prefix}_{k-1}(\tilde{x})=u\,]$, $\sigma_v(x)=\mathbf{1}[\,\mathrm{suffix}_{k-1}(\tilde{x})=v\,]$.
By the classical characterization of LT$_k$, membership depends only on this predicate vector, hence some $S_L$ exists.

\paragraph{$\mathsf{PT}_m$.}
Let $P=\{r_u: u\in\bigcup_{\ell=1}^{m}(\Sigma\cup\{\#\})^\ell\}$ with $r_u(x)=\mathbf{1}[\,u\sqsubseteq\tilde{x}\,]$.
By Simon’s theorem \citep{simon1975}, membership depends only on the set of subsequences of length $\le m$, hence some $S_L$ exists.

\paragraph{$\mathsf{LTT}_{k,\boldsymbol\tau}$.}
Let $P=\{c_{g,t}: |g|\le k,\ t\in\{1,\dots,\tau_g\}\}$ with $c_{g,t}(x)=\mathbf{1}[\,\#\{g\text{ in }\tilde{x}\}\ge t\,]$, optionally augmented with thresholded prefix/suffix indicators of length $k{-}1$.
Counts saturate at $\tau_g$, so $P$ is finite and determines membership.

\paragraph{$\mathsf{TSL}_k$.}
Let $\mathcal{T}=\{T\subseteq\Sigma\}$ be the (finite) set of tiers.
For each $T\in\mathcal{T}$ and $g\in(T\cup\{\#\})^k$, include $s_{T,g}(x)=\mathbf{1}[\,g\prec \pi_T(\tilde{x})\,]$.
If $L$ is specified by $(T,F_T)$ (tier $T$ and forbidden tier $k$-grams), then membership is equivalent to $s_{T,g}(x)=0$ for all $g\in F_T$, i.e., $(s_{T,g}(x))\in S_L$ for a suitable $S_L$.

\paragraph{Unions of parameterized subclasses.}
Families like $\mathsf{PT}=\bigcup_m \mathsf{PT}_m$ or $\mathsf{SL}=\bigcup_k \mathsf{SL}_k$ are directed unions of finitely observable subclasses.
Unless parameters are bounded, no single finite predicate set suffices for the whole union; this does not affect applications where the parameters are fixed by design.

\bibliography{references2}
\bibliographystyle{acl_natbib}

\end{document}